\documentclass[letterpaper, 10 pt, conference]{ieeeconf}  

\IEEEoverridecommandlockouts                            

\usepackage[T1]{fontenc}
\usepackage{graphics} 
\usepackage{epsfig} 
\usepackage{mathptmx} 
\usepackage{times} 
\usepackage{amsmath} 
\usepackage{amssymb} 
\usepackage{hyperref}
\usepackage{cleveref}
\usepackage{cite}
\usepackage{todonotes}
\usepackage{tikz}
\usepackage{siunitx}

\usetikzlibrary{arrows.meta,positioning,calc,shapes.geometric}

\crefname{figure}{Fig.}{Figs.}
\Crefname{figure}{Fig.}{Figs.}

\newtheorem{theorem}{Theorem}
\newtheorem{lemma}{Lemma}
\newtheorem{corollary}{Corollary}

\newtheorem{conjecture}{Conjecture}
\newtheorem{definition}{Definition}

\newtheorem{remark}{Remark}
\newcommand{\R}{\mathbb{R}}
\newcommand{\C}{\mathbb{C}}
\newcommand{\T}{^{\sf T}}

\title{\vspace{-1.2cm}\LARGE \bf
Pass the Bucket: Efficient, Robust, Local\\
Load Balancing for Teams of Heterogeneous Robots
}

\author{Tobias Wallner$^{1}$, Dominik Krupke$^{1}$, Arne Schmidt$^{1}$, and Sándor P. Fekete$^{1,2}$
\thanks{This work was supported by DFG project „Space Ants`` (FE 407/22-1).}
\thanks{$^{1}$Department of Computer Science, TU Braunschweig, Braunschweig, Germany. \texttt{\{wallner, krupke, aschmidt\}@ibr.cs.tu-bs.de}, \texttt{s.fekete@tu-bs.de}. \quad
$^{2}$L3S Research Center, Germany}
\thanks{This paper was submitted to IROS 2026 on March 2nd and accepted on June 17th.}
\thanks{© 2026 IEEE.  Personal use of this material is permitted.  Permission from IEEE must be obtained for all other uses, in any current or future media, including reprinting/republishing this material for advertising or promotional purposes, creating new collective works, for resale or redistribution to servers or lists, or reuse of any copyrighted component of this work in other works.}
}

\begin{document}

\maketitle
\thispagestyle{empty}
\pagestyle{empty}

\begin{abstract}
We study the problem of decentralized, self-organized task sharing
for a swarm of heterogeneous robots that collaborate in transportation or other objectives that 
require coordinated motion planning. 
To this end, we present theoretical and practical results for the simple but
effective mechanism of \emph{bucket brigades} for load balancing,
in which a team of heterogenous robots share a spatial task in a confined, one-dimensional space,
while only being able to sense collisions with neighbors or walls. The
goal is to optimize throughput of the overall system, without central control
or information, aiming at an interval partition proportional to robot velocities.
We address possible chaotic system behavior by developing a 
stabilization mechanism based on
simple local aid, a ``token'', that temporarily decelerates robots after an encounter.
This purely local change eliminates persistent oscillations, resulting
in convergence towards a stable system state. 
We accelerate system convergence
by comparing a single boundary token to ubiquitous two-directional tokens and
optimizing the deceleration factor. Event-driven
simulations report convergence times and robustness: For a large variety of
perturbations (such as robot deletion, position or velocity jittering), 
the system reliably re-converges. The results suggest a local,  
practical mechanism for robust load balancing for heterogeneous teams of robots
that promises an effective tool as basis for more complex scenarios.
\end{abstract}

\section{Introduction}

Multi-robot coordination promises efficiency through parallelism and robustness through redundancy, but faces challenges in communication overhead, centralized control scalability, and resilience to perturbations—especially for heterogeneous teams of autonomous mobile robots (AMRs) with minimal sensing capabilities.

We study \emph{collision-driven bucket brigades} for load balancing on a one-dimensional interval: robots move back and forth, reversing direction upon collision with neighbors or boundaries (\Cref{fig:ex1}). 
Such one-dimensional coordination arises in confined robotic environments including warehouse aisles, agricultural rows, and infrastructure corridors (see \Cref{sec:applications}).
Classical Operations Research has shown that bucket brigades can optimize assembly-line throughput under appropriate worker ordering \cite{bartholdi1996production}. 
Rather than relying on such ordering, we study how collision-driven dynamics can self-organize territory partitioning: 
each robot covers a share of the interval proportional to its speed, achieving balanced spatial load distribution.
In \Cref{sec:applications}, we present examples where the shared interval represents a physical domain; more generally, it could also model stages of a sequential task.

\begin{figure}
    \centering
    \includegraphics[width=\columnwidth]{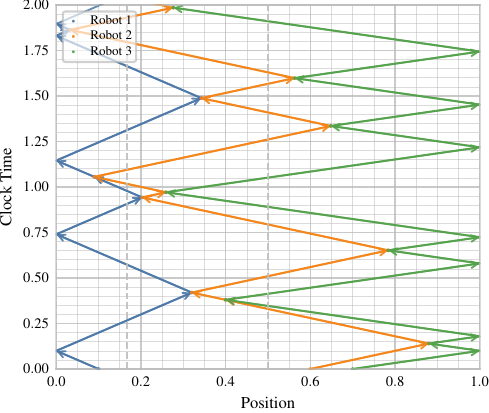}    \vspace{-.6cm}
    \caption{Example configuration: 3 robots; velocities 1, 2, 3; starting positions 0.1, 0.6, 0.7; starting directions left, right, right. The dashed gray vertical lines indicate the equilibrium borders of the robot pairs.}
    \label{fig:ex1}
\end{figure}

The challenge lies in minimal individual capabilities: robots sense only local collisions and can reverse direction. 
They have no shared clock, no communication beyond physical contact, and minimal memory (one stored scalar for damping). 
Such requirements suit resource-constrained platforms and infrastructure-limited environments, as discussed in \Cref{sec:applications}.
Despite these constraints, the collective must autonomously form a balanced steady state and recover from disturbances such as velocity changes or robot additions and removals.
For heterogeneous velocities, the desirable equilibrium partitions the line into subintervals proportional to velocities, with neighboring robots meeting at fixed boundaries in synchronized cycles. 
Achieving this from arbitrary initial conditions is nontrivial: the dynamics are intricate and, even with two identical robots, indefinitely oscillating.

\paragraph*{Contributions}
We develop an event-based analysis of collision dynamics, showing that undamped systems exhibit conservative behavior that prevents asymptotic convergence. 
We introduce a \emph{token damping mechanism}—storing a single collision position per robot and decelerating upon crossing—that reliably induces convergence in simulation. 
Perturbation experiments demonstrate robustness to robot addition/removal, velocity changes, and localization noise. 
Finally, we discuss potential application domains, considerations on physical deployment, and alternative approaches.

\section{Related Work}

Bucket brigades originate in Operations Research for assembly lines. Bartholdi and Eisenstein \cite{bartholdi1996production,bartholdi1999dynamics} showed that deterministic lines with workers ordered by increasing velocity and infinite backward speed self-balance. 
Extensions address throughput maximization \cite{Lim2009_MaximizingThroughput}, stochastic task times \cite{bartholdi2001performance,Peng2022_StochasticDiscrete}, and tree topologies \cite{Bartholdi2006_Trees}. 
Two-directional (cellular) variants are studied in \cite{lim2011cellular,Lim2014_Cellular}.
Finite backward motion complicates convergence: equal finite backward velocity preserves self-balancing \cite{bratcu2009some}, while distinct forward/backward velocities require ordering by inverse-velocity differences \cite{bartholdi2004chaos} (not applicable here, as we assume identical bidirectional velocities). 
\cite{MensingThesis} shows that one-directional damping between the first pair ensures convergence for homogeneous velocities—an idea we generalize to heterogeneous teams.

\emph{Fence patrolling} \cite{czyzowicz2011boundary,fence_kawamura_2015,chen2013fence} addresses related questions: agents with different velocities minimize point latency on a line segment. 
The canonical periodic schedule partitions the line proportionally to velocities, mirroring our equilibrium.

Robotic bucket brigade deployments remain sparse: foraging \cite{adaptive_lein_2008,Ostergaard2001_EmergentBucketBrigading}, assembly \cite{Klavins2000_ConcurrentRobot}, and transport \cite{Itani1995_MiniRobots}. 
These works demonstrate feasibility in principle but lack systematic analysis of heterogeneous teams, damping mechanisms, and convergence---which we address in this work.

Relevant studies regarding potential robotic applications as well as alternative approaches to load balancing and coordination in multi-robot systems are discussed in \Cref{sec:applications}.

\section{Preliminaries}

\paragraph*{Setting and objective}
We consider $n$ robots on the line segment $[0,1]$. Robot $i$ has constant velocity magnitude $v_i>0$, a position $p_i(t)\in[0,1]$ with $p_i(0) \neq p_j(0)$ for all $i\neq j$, and a direction $d_i(t)\in\{-1,+1\}$. 
Robots move at constant velocity between discrete events and reverse only upon interactions:
(i) head-on neighbor encounters cause both to reverse;
(ii) same-direction catch-ups make only the faster one reverse;
(iii) wall contacts at $0$ or $1$ cause a reversal.
Left-to-right order is preserved (no overtakes), and we assume that in case of degeneracy, i.e., multiple event coinciding, collision events are handled ordered by their index.

\paragraph*{Event-based state-collision abscissas}
For each adjacent pair $(i,i{+}1)$, let $x_i(k)\in[0,1]$ be the position of their $k$-th collision
with boundary markers $x_0:=0$, $x_n:=1$.
We focus on the rotational regime in which only alternating head-on encounters occur; then every adjacent pair collides infinitely often and the sequences $\{x_i(k)\}$ are well defined.
Empirically, the resulting $x_i(k)$ trace smooth, almost-sinusoidal patterns; see \Cref{fig:ex2,fig:ex4}.

\paragraph*{Equilibrium partition}
The goal is a balanced equilibrium partition in which neighbors meet at fixed points and the subinterval lengths are proportional to velocities.
The equilibrium points are
$
x_i^* = \frac{\sum_{j=1}^{i} v_j}{\sum_{j=1}^n v_j}\,, \; i=1,\dots,n-1.
$

\begin{figure}
    \centering
    \includegraphics[width=\columnwidth]{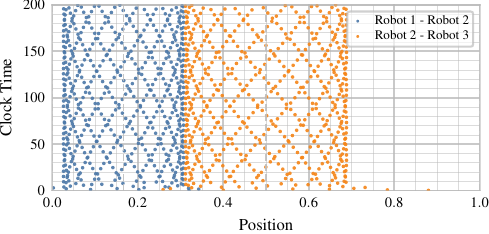}
    \caption{Collision positions for the configuration of \Cref{fig:ex1}. Dashed lines: equilibrium borders $x_i^*$.}
    \label{fig:ex2}
\end{figure}

\begin{figure}
    \centering
    \includegraphics[width=\columnwidth]{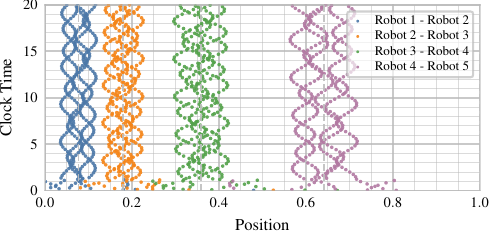}
    \caption{Collision positions for another start configuration of 5 robots. Dashed lines: equilibrium borders $x_i^*$.}
    \label{fig:ex4}
\end{figure}

As we will show in the following section, the system will never converge to equilibrium without additional damping.
To foreshadow our damping mechanism, consider two robots with velocities $v_1,v_2>0$ and let $p_k$ be their $k$-th meeting point. 
We induce damping by slowing down a robot by a constant factor upon crossing $p_k$. 
To illustrate this, drop a \emph{token} at $p_k$: 
the first robot returning to $p_k$ (after its wall bounce) picks it up and moves at velocity $\alpha v_i$ with $0<\alpha<1$; the other keeps full velocity.
Upon collision the token is dropped again. 

\begin{lemma}[Affine contraction toward $x^*$]
\label{lem:2rob2dir}
Let $x^*:=\frac{v_1}{v_1+v_2}$. 
Wlog assume $p_k<x^*$ (else swap indices). 
Then for each~$\alpha\in(0,1)$, there is $\kappa\in (-1,1)$ such that 
$
p_{k+1}-x^*=\kappa\,(p_k-x^*)$.
\end{lemma}

\begin{proof}
With robot~1 returning first, equate travel times
\[
\underbrace{\tfrac{2p_k}{v_1}+\tfrac{p_{k+1}-p_k}{\alpha v_1}}_{\text{carrying token (slow only on last leg)}}
=\underbrace{\tfrac{(1-p_k)+(1-p_{k+1})}{v_2}}_{\text{other}},
\]
multiply by $\alpha v_1 v_2$ and solve for $p_{k+1}$:
\[
p_{k+1}
=\frac{2\alpha v_1 + \bigl(v_2(1-2\alpha)-\alpha v_1\bigr)p_k}{\,v_2+\alpha v_1\,}.
\]
Recenter at $x^*=\tfrac{v_1}{v_1+v_2}$ to obtain
\[
p_{k+1}-x^*=\kappa\,(p_k-x^*),\quad
\kappa=\frac{v_2(1-2\alpha)-\alpha v_1}{\,v_2+\alpha v_1\,}\in(-1,1).
\]
Hence, with each collision, the abscissa $p_k$ is pulled closer to the equilibrium $x^*$, and we converge.
\end{proof}

\section{Mathematical Analysis}
Now we provide a rigorous algebraic analysis of collision dynamics, using an affine map to describe the update from one set of collision abscissas to the following one. By proving key properties of this map, we show that the undamped system is stable and cannot converge to the equilibrium unless started there. We then give an outlook to the analysis of the damped case with a single or several tokens.
\subsection{Original System}
We track only discrete robot--robot and robot--wall collisions on $[0,1]$ for $n$ robots with velocities $v_1,\dots,v_n>0$.
Let $m:=n-1$ denote the number of adjacent pairs and let $x_i(k)\in[0,1]$ be the abscissa of the $k$-th collision of pair $(i,i{+}1)$, with boundary markers $x_0:=0$ and $x_n:=1$.

In the initial phase, two types of robot-robot collisions can occur: head-on encounters, where both robots reverse direction, and same-direction catch-ups, where only the faster robot reverses. Empirical observations indicate that, after a while, the system enters a \emph{rotational regime}, where only head-on collisions of robots with different movement directions occur. As the desired equilibrium and any state that is close to it fulfill this rotational property, we focus on this catch-up-free regime in the following analysis.

Collect the most recent collision positions into $x=(x_1,\dots,x_m)^\top$ with $x_0=0$ and $x_n=1$.
One alternating sweep of head-on encounters maps $x\mapsto x'$ through a sparse, neighbor-coupled system.
As illustrated in \Cref{fig:collision2}, equating travel times along the two closed quadrilaterals around pair $(i,i{+}1)$ yields the odd/even relations
\begin{align*}
\text{odd }i:\, \,
&x'_i
  =-\,x_i
  +2\,\frac{v_i}{v_i+v_{i+1}}\,x_{i+1}
  +2\,\frac{v_{i+1}}{v_i+v_{i+1}}\,x_{i-1},\\
\text{even }i:\, \,
&(v_i+v_{i+1})\,x'_i
  -2\,v_{i+1}\,x'_{i-1}
  -2\,v_i\,x'_{i+1}
  =-(v_i+v_{i+1})\,x_i.
\end{align*}

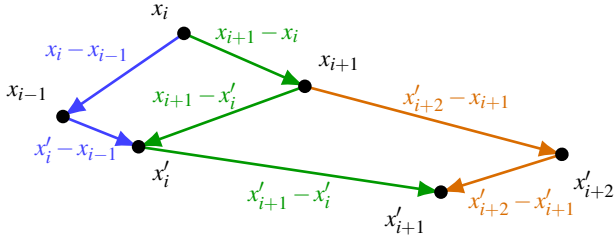
\begin{figure}
    \centering
    \begin{tikzpicture}[
  >=Latex,
  blueStep/.style={blue!75, line width=1.0pt},
  greenStep/.style={green!60!black, line width=1.0pt},
  orangeStep/.style={orange!85!black, line width=1.0pt},
  ref/.style={gray!70, dashed, line width=0.9pt},
  dotGray/.style={circle, inner sep=1.8pt, fill=black, draw=none},
  diaBlue/.style={diamond, draw=blue!75, fill=blue!20, minimum size=4pt, inner sep=1.4pt},
  triGreen/.style={regular polygon, regular polygon sides=3, draw=green!60!black, fill=green!20, inner sep=1.6pt},
  triOrange/.style={regular polygon, regular polygon sides=3, draw=orange!85!black, fill=orange!20, inner sep=1.6pt},
  lbl/.style={font=\small}
]

\coordinate (xm1)   at (0.00, 1.00);   
\coordinate (xi)    at (1.60, 2.10);  
\coordinate (xip)   at (1.00, 0.60);   
\coordinate (xp1)   at (3.20, 1.40);  
\coordinate (xip1p) at (5.00, 0.00);   
\coordinate (xip2p) at (6.60, 0.50);  
\coordinate (xp2)   at (6.60, 3.20);

\draw[blueStep, -{Latex[length=3mm]}] (xi) -- node[lbl, above left, xshift=6pt, yshift=1pt] {$x_i - x_{i-1}$} (xm1);
\draw[blueStep, -{Latex[length=3mm]}] (xm1) -- node[lbl, below=2pt, xshift=-9pt, yshift=4pt] {$x'_i - x_{i-1}$} (xip);

\draw[greenStep, -{Latex[length=3mm]}] (xi) -- node[lbl, above=2pt, xshift=5pt] {$x_{i+1} - x_i$} (xp1);
\draw[greenStep, -{Latex[length=3mm]}] (xp1) -- node[lbl, above=1pt, xshift=-10pt, yshift=-2pt] {$x_{i+1} - x'_i$} (xip);
\draw[greenStep, -{Latex[length=3mm]}] (xip) -- node[lbl, below=1pt] {$x'_{i+1} - x'_i$} (xip1p);

\draw[orangeStep, -{Latex[length=3mm]}] (xp1) -- node[lbl, above=2pt, xshift=9pt, yshift=-3pt] {$x'_{i+2} - x_{i+1}$} (xip2p);
\draw[orangeStep, -{Latex[length=3mm]}] (xip2p) -- node[lbl, below=2pt, xshift=8pt, yshift=0pt] {$x'_{i+2} - x'_{i+1}$} (xip1p);

\node[dotGray, label={[lbl, black]above left:$x_{i-1}$}] at (xm1) {};

\node[dotGray, label={[lbl, black]above left:$x_{i}$}] at (xi) {};
\node[dotGray, label={[lbl, black]below right:$x'_{i}$}] at (xip) {};

\node[dotGray, label={[lbl, black]above right:$x_{i+1}$}] at (xp1) {};
\node[dotGray, label={[lbl, black]below left:$x'_{i+1}$}] at (xip1p) {};
\node[dotGray, rotate=-20, label={[lbl, black]below right:$x'_{i+2}$}] at (xip2p) {};

\end{tikzpicture}
    \caption{Closed time-matching quadrilaterals around pair $(i,i{+}1)$ (shown for odd $i$) yield the two linear relations for $(x_i',x_{i\pm1}')$ in one alternating sweep.}
    \label{fig:collision2}
\end{figure}

Stacking these $m$ relations yields a three-banded system
$
A\,x' \;=\; B\,x \;+\; c
\,\Leftrightarrow \,
x' = A^{-1}B\,x + A^{-1}c,
$
where $A$ and $B$ reflect the coefficients of $x'$ and $x$ in the odd/even relations, and $c$ collects the constants.
Let $x^*$ denote the velocity-proportional equilibrium partition defined in Section~III.
With $u:=x-x^*$, the centered dynamics are linear:
$
u' \;=\; M\,u,
$ with $
M:=A^{-1}B\in\mathbb{R}^{m\times m}.
$
Note that $x^*$ is the unique fixed point of the affine map $x\mapsto A^{-1}B x + A^{-1}c$, and, equivalently, $u\mapsto M u$ has the unique fixed point $u=0$.
Now, we show that there exists a quadratic invariant of the dynamics, which allows us to interpret the motion as a rigid rotation on ellipsoids.

\begin{lemma}[Rotation radius (G-isometry)]
\label{lem:radius}
There exists a symmetric positive definite tridiagonal matrix $G\in\mathbb{R}^{m\times m}$ such that
$M^{\T} G M = G$. 
\end{lemma}

\begin{proof}
We give a sketch of the proof; some details and explicit computations are omitted due to space constraints.

\emph{Local update.}
In centered coordinates $u$ (with $u_0=u_{m+1}=0$), the two time-matching quadrilateral equalities for pair $i$ (the ``odd'' and ``even'' cases shown above) both reduce---after subtracting the equilibrium and eliminating $x^*$---to the same linear row update of the form
$
u'_i \;=\; -\,u_i \;+\; a_i\,u_{i+1} \;+\; b_i\,u_{i-1},
$
i.e., replacing only the $i$-th coordinate by this combination of its two neighbors; all other $u_j$ remain unchanged.
Define for each $i\in\{1,\dots,m\}$, the linear single-event operator $E_i\in\R^{m\times m}$ such that $(E_i u)_j = -u_i + a_i u_{i+1} + b_i u_{i-1}$ for $j=i$ and $(E_i u)_j =u_j$ otherwise.
Thus a single head-on event at pair $i$ applies $u\mapsto E_i u$.

\emph{Choice of $G$.}
We now construct a symmetric tridiagonal matrix $G$ with diagonal entries $g_i>0$ and off-diagonals $h_i<0$ ($G_{i,i}=g_i$, $G_{i,i+1}=G_{i+1,i}=h_i$) that makes every single-event operator $E_i$ a $G$-isometry: $E_i^{\T} G E_i=G$ for all $i$.
Because $E_i$ alters only coordinate $i$, the change in $Q(u):=u^{\T}Gu$ depends only on the $3\times3$ block involving indices $i-1,i,i+1$.
We choose $h_i$ and $g_i$ such that
$E_i^{\T} G E_i = G$ for all $i$.

\emph{Positive definiteness} (so $Q(u)=u^{\T} G u$ is a norm).
Choose any scale $g_m>0$, set $g_i=\frac{b_{i+1}}{a_i}g_{i+1}$ (backwards), then $h_i=-(a_i/2)g_i$; hence all $g_i>0$ and $h_i<0$.
We use a standard continuant for determinants of tridiagonal matrices
and inductively show that
all leading minors are positive and $G\succ0$.

\emph{Single event invariance.}
Event $i$ changes $u_i$ to $-u_i+a_i u_{i+1}+b_i u_{i-1}$.
The change in $Q(u)$ depends only on $u_{i-1}$, $u_i$, $u_{i+1}$.
$G$ is chosen such that $Q$ is preserved:
$E_i^{\T}GE_i=G$.

\emph{One sweep as product.}
Let
$E_{\text{odd}}=\prod_{i\ \text{odd}}^{\nearrow}E_i$,
$E_{\text{even}}=\prod_{i\ \text{even}}^{\nearrow}E_i$.
Odd rows update using original neighbors; then even rows use the already updated odd neighbors---exactly the two-phase physical sweep derived from $A x'=B x + c$. A direct row-wise check of $A x'=B x + c$ shows: each odd row is satisfied after the first phase, each even row after substituting the updated odd neighbors; by invertibility of $A$ this forces $x'=E_{\text{even}}E_{\text{odd}}x=Mx$. Thus $M=E_{\text{even}}E_{\text{odd}}$.

\emph{Global invariance.}
Using transposition reversal and single-event invariance:
$
  M^{\T}GM =(E_{\text{even}}E_{\text{odd}})^{\T}G(E_{\text{even}}E_{\text{odd}})
=E_{\text{odd}}^{\T}(E_{\text{even}}^{\T}GE_{\text{even}})E_{\text{odd}}=G.
$
Thus $M$ is a $G$-isometry.
\end{proof}

\begin{corollary}[Rotation form]
\label{cor:rotform}
Each trajectory lies on the invariant ellipsoid $\{u:u^{\T}Gu=R^2\}$. 
\end{corollary}

\begin{proof}
  The product $u^{\T}Gu$ is invariant under $u\mapsto Mu$: $(Mu)^{\T}G(Mu)=u^{\T}(M^{\T}GM)u=u^{\T}Gu$ by \Cref{lem:radius}.
\end{proof}

\begin{remark}[Geometric note]
The map is a composition of planar rotations in the $G$-inner product. 
\Cref{fig:ex3} illustrates this rotational evolution for the configuration of \Cref{fig:ex1}.
\end{remark}

We now formalize the non-convergence consequence.

\begin{theorem}[No convergence without equilibrium start]
\label{prop:noconv}
In the rotational regime, $u_k \to 0$ if and only if $u_0=0$.
\end{theorem}
\begin{proof}
From \Cref{lem:radius}, $u_k^{\T} G u_k = u_0^{\T} G u_0$ for all $k$. If $u_0\neq 0$, then $u_k^{\T} G u_k$ is a fixed positive constant, so $u_k$ cannot converge to $0$. Conversely, if $u_0=0$, then $u_k= 0$.
\end{proof}

\begin{figure}
    \centering
    \includegraphics[width=0.8\columnwidth]{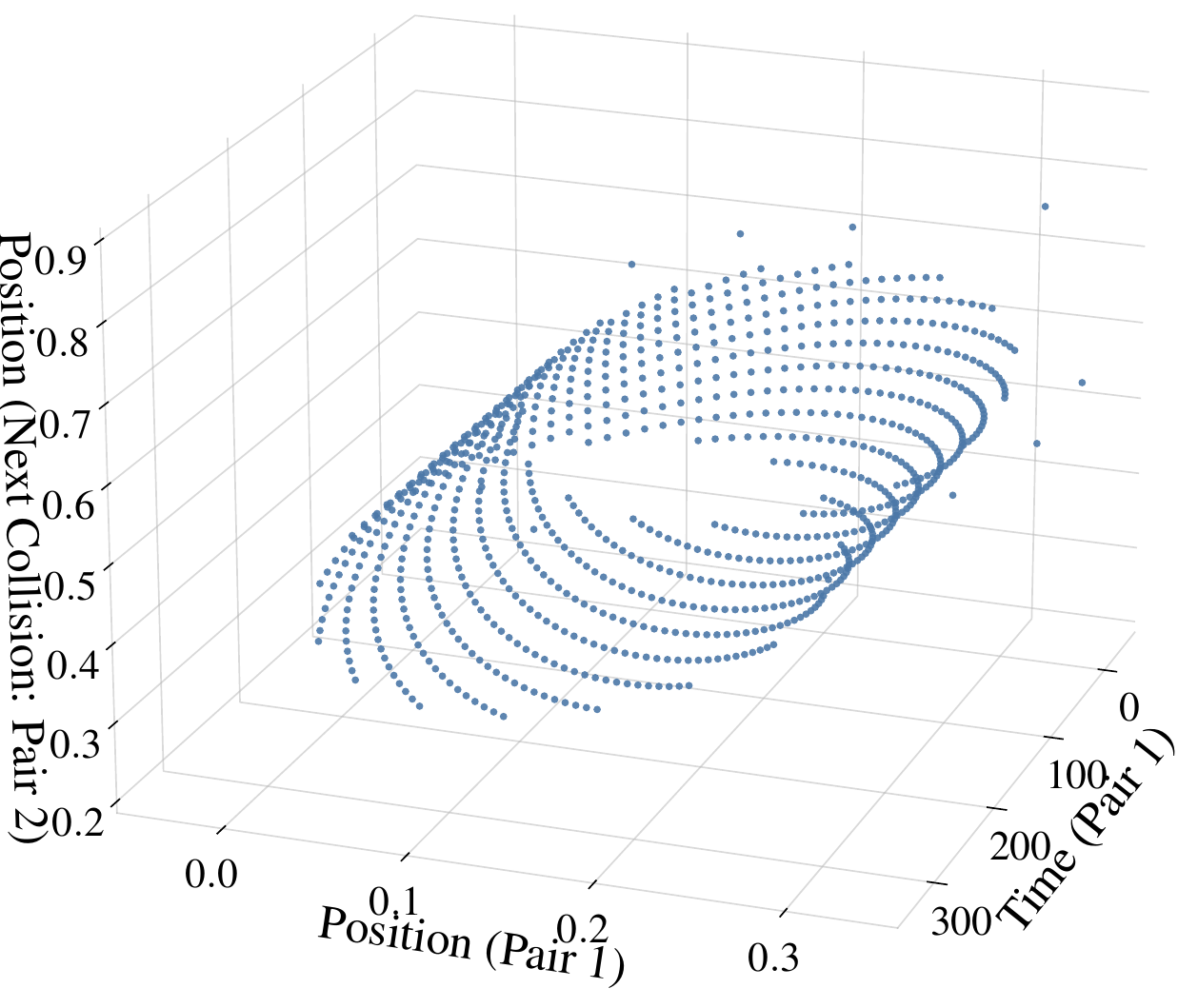}
    \caption{Collision positions of both pairs for exemplary start configuration of \Cref{fig:ex1}. This illustrates the rotational evolution of $(x,y)$.}
    \label{fig:ex3}
\end{figure}

\begin{remark}
  \Cref{lem:radius} and complementing computations additionally imply that $M$, its characteristic polynomial $\chi_M$ and its spectrum $\sigma(M)$ have these algebraic properties:

  (i) $\det M = (-1)^m$, showing that $M$ is volume-preserving for $m$ even, and a composition of a volume-preserving map and a reflection for $m$ odd.
(ii) All eigenvalues of $M$ lie on the unit circle, establishing that $M$ represents a composite of planar rotations and causes oscillating collision abscissas.
  (iii) $\chi_M$ is \emph{self-reciprocal}:
      $
          \chi_M(\lambda)=\lambda^{\,m}\,\chi_M\!\bigl(\tfrac1\lambda\bigr)
          $ with $\lambda\in\C.
      $
      The eigenvalues appear in conjugate pairs $\{\lambda,\overline\lambda\}$, plus
      the single real number $-1$ if $m$ is odd.
      (iv) $-1$ is an eigenvalue if and only if $m$ is odd (and then
      with multiplicity one).
\end{remark}

\subsection{Damped System: Token Mechanics}
The undamped rotational transfer $u'\!=\!Mu$ is conservative and yields no contraction toward the equilibrium.

To induce damping, we introduce a \emph{token}: a stored collision abscissa between adjacent robots. When a robot crosses the token, it slows by factor $\alpha\in(0,1)$ until its next reversal. At a collision with another robot, the token is dropped. This modifies only the local timing of the adjacent pair. A token always remains between its original robot pair.

\newcommand{\pebOne}{\texttt{Token-1}}

\subsubsection{Single, one-directional token ("\pebOne")}

As the simplest version of this technique, place a single token between robots $1$ and $2$ that only affects the velocity of robot 1; robot 2 also picks up and drops the token, but without adjusting its velocity. We call this token \emph{one-directional}, as it only impacts one robot.
This changes only the first pair's relation:
$
\frac{2x_1}{v_1}+\frac{x'_1-x_1}{\alpha v_1}
   \;=\;
   \frac{x_2-x_1}{v_2}+\frac{x_2-x'_1}{v_2},
$
all other equations remain unchanged.
In matrix form, with the same $A$, replace $B$ by $B_\alpha$ 
and obtain
$
u' \;=\; M_\alpha\,u,
$ with $
M_\alpha := A^{-1}B_\alpha.
$

$M_\alpha$ now represents the collision abscissa update for a step where the first robot actually picks up and carries the token. 
Otherwise,
the original $M$ applies.
We aim at following an argument similar to the undamped case, but now $M_\alpha$ is not a $G$-isometry anymore. As a first step towards convergence, we show that the determinant contracts.

\begin{lemma}[Determinant contraction with one token]
\label{lem:token-det}
For each $0<\alpha<1$ we have $|\det M_\alpha|<1$.
\end{lemma}

\begin{proof}[Proof sketch]
Only the first row of $B$ is modified.
Express $\det(A^{-1}B_\alpha)=\det(B_\alpha)/\det(A)$ and evaluate the continuant for $\det(B_\alpha)$ by Laplace expansion anchored at the modified row; the unchanged tail cancels with $\det(A)$.
The resulting rational form reduces to the expression, 
$
\det M_\alpha
= (-1)^{m-1}\,
\frac{v_2(1-2\alpha)-\alpha v_1}{\,v_2+\alpha v_1\,},
$
with $|\det M_\alpha|<1$ for $0<\alpha<1$.
\end{proof}

The main convergence result requires a full spectral contraction proof, which we leave as an open problem in the theoretical part. In \Cref{sec:simulations}, we present simulation results that strongly support the following conjecture.

\begin{conjecture}[Spectral contraction]
\label{conj:spectral}
\emph{The damped transfer map $M_\alpha$ is repeatedly applied and its spectral radius satisfies $\rho(M_\alpha)<1$, such that $u\to 0$ for $0<\alpha<1$.}
\end{conjecture}

 A formal contraction proof remains open but does not impede practical deployment (\Cref{sec:applications,sec:simulations}).

\Cref{fig:ex2_token} illustrates 
the effect of such a token placed for the exemplary start configuration of \Cref{fig:ex1}, inducing convergence of the collision abscissas towards the equilibrium.

\begin{figure}
    \centering
    \includegraphics[width=0.5\textwidth]{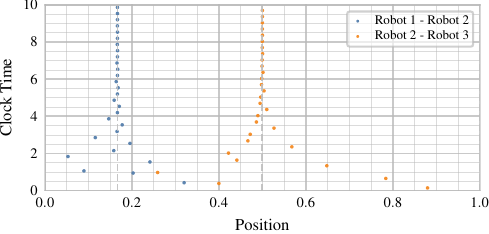}
    \caption{Collision positions for exemplary start configuration of \Cref{fig:ex1} with one-directional token in-between the first pair (\pebOne). The system approaches the equilibrium quickly.}
    \label{fig:ex2_token}
\end{figure}

\newcommand{\pebAll}{\texttt{Token-all}}

\subsubsection{Ubiquitous tokens ("\pebAll")}

Finally, we consider placing one token between any subsequent pair. This would correspond to every robot having the ability to store its last collision abscissa with both adjacent robots and to alter its velocity accordingly.

\section{Simulation Experiments} 
\label{sec:simulations}

To complement the theoretical analysis, we turn to simulations to develop empirically grounded hypotheses about aspects of the system that remain analytically unresolved.
This leads us to the following research questions:

\begin{description}
  \item[RQ1]\label{rq:rotational_entry} Do all systems always converge to the rotational, catch-up-free regime?
  \item[RQ2]\label{rq:pebble_placement} How do convergence speed and tail behavior depend on the initial token placement?
  \item[RQ3]\label{rq:speed_heterogeneity} How does heterogeneity in agent velocities influence convergence?
  \item[RQ4]\label{rq:pebble_slowdown} How does the token slow-down factor $\alpha$ impact convergence dynamics?
  \item[RQ5]\label{rq:reconvergence} How does the system behave after different types of perturbations?
\end{description}

\subsection{Experiment Design}

Our approach relies on Monte Carlo simulations across randomized trials, where we vary key parameters and record relevant metrics to address each research question.  
To conduct these experiments, we implemented a Python-based simulator that employs high precision arithmetic through the \texttt{decimal} module. 
Since velocities change only at predictable collision events, we adopt an event-driven simulation strategy, ensuring both efficiency and accuracy.
Metric tracking and termination checks are restricted to these events.  

We typically perform $N=1000$ trials, each involving up to $n=9$ robots initialized with random positions in $[0,1]$, random directions, and velocities sampled uniformly from $[0.5,10]$.
Comparisons are made between settings with no tokens, a single one-directional token between the first pair (\pebOne), and two-directional tokens between all adjacent pairs (\pebAll).
For most experiments, we utilize empirical complementary cumulative distribution functions (CCDFs), which plot the fraction of trials that have not yet satisfied a specified criterion (e.g., convergence) by time $t$.
The full set of experiments, including the simulator, will be made available in an open-source repository.

\begin{definition}[Convergence criterion]\label{def:convergence_criterion}
  We say the system is converged if, for every adjacent robot pair, the previous 10 collision abscissas have not been more than $\varepsilon=5\times10^{-4}$ apart from the respective equilibrium abscissa.
\end{definition}

\paragraph*{\hyperref[rq:rotational_entry]{RQ1}}

As it is impossible to empirically validate that a system will remain catch-up free indefinitely, we instead demonstrate that the ratio of catch-up events rapidly drops to zero, with no tested run exhibiting a same-direction catch-up beyond a certain time. The ratio is computed per time step, with the step size increasing exponentially to scale with the logarithmic plot.
We compare the system with no tokens, \pebOne, and \pebAll ~(slow-down factor $\alpha=0.5$).

\paragraph*{\hyperref[rq:pebble_placement]{RQ2}}

We conduct $N$ randomized trials for both token variants and for $n \in \{3,5,7\}$. 
The convergence times (as defined in \Cref{def:convergence_criterion}) are then compared using CCDFs.

\paragraph*{\hyperref[rq:speed_heterogeneity]{RQ3}}

We sweep the upper bound of the uniformly sampled velocities (keeping the lower bound fixed at 0.5) and plot the resulting convergence CCDFs (illustrated for $n=5$ with \pebAll) to observe the impact of higher maximum velocities on collision frequency and balancing.

\paragraph*{\hyperref[rq:pebble_slowdown]{RQ4}}

The influence of the slow-down factor $\alpha$ is assessed by measuring the average convergence time for $n \in \{3,5,7,9\}$ with \pebAll.
For each $n$, $\alpha$ is varied over $(0,1)$ in increments of $0.01$, using the same $N$ samples for all $\alpha$. The resulting mean convergence time curves are plotted as functions of $\alpha$.
The minimum of each curve indicates the empirically optimal $\alpha$ for that $n$.  

\paragraph*{\hyperref[rq:reconvergence]{RQ5}}

Starting from an equilibrium-like state (equilibrium abscissas, alternating directions), we apply perturbations:
remove or add one randomly chosen robot, jitter the velocity of one randomly chosen robot by factor $0.5$ or $2$, and jitter the positions of one randomly chosen robot or all robots.
We form reconvergence CCDFs ($n\in\{3,5,7\}$, both token settings, $\alpha=0.5$).
Additionally, we sweep the position jitter amplitude from $0\%$ to $25\%$ of the interval length and plot mean convergence time as a function of amplitude for $n\in\{3,5,7,9\}$ with \pebAll.
For velocity perturbations, we compare reconvergence times to random-start baselines using one-sided Mann--Whitney~U tests with Bonferroni correction ($\alpha_\text{corr}=0.0125$, 4 independent configurations).

\subsection{Results}

\paragraph*{\hyperref[rq:rotational_entry]{RQ1}} 

For all tested $n\in\{3,5,7\}$ ($n=5$ shown in \Cref{fig:exp_rotational}), the ratio of catch-up events to total collisions per clock time step decreases rapidly, reaches zero eventually, and does not recover within the extensive simulation timeframe. 
Tokens both reduce the initial ratio and accelerate the decline. 

\begin{figure}[t]
  \centering
  \includegraphics[width=\columnwidth]{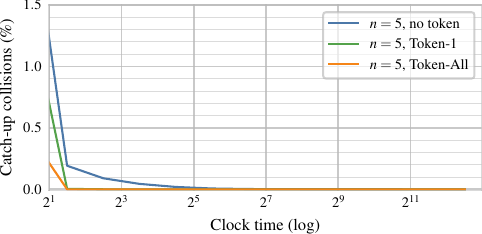}
  \caption{Ratio of catch-up events to total collisions per clock time bucket, here shown for $n=5$ (other $n$ look similar). Even without tokens, the system reaches the rotational, catch-up free state quickly.}
  \label{fig:exp_rotational}
  \end{figure}

\paragraph*{\hyperref[rq:pebble_placement]{RQ2}} 

The CCDFs in \Cref{fig:exp_convergence} show that systems with tokens between all pairs converge more quickly, with little variation across $n \in \{3,5,7\}$.
In contrast, with \pebOne, the convergence time increases noticeably as $n$ grows.

\begin{figure}[t]
  \centering
  \includegraphics[width=\columnwidth]{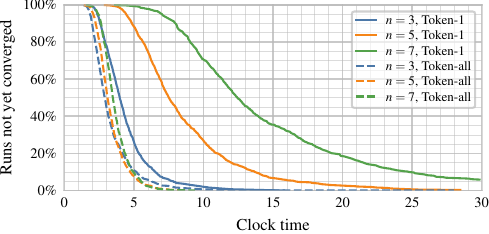}
  \caption{Empirical CCDFs for $n\in\{3,5,7\}$. \pebAll \ induces faster convergence than just \pebOne. More robots take longer to converge.}
  \label{fig:exp_convergence}
  \end{figure}

\paragraph*{\hyperref[rq:speed_heterogeneity]{RQ3}} 

As shown for $n=5$ with \pebAll \ in \Cref{fig:exp_hetero}, increasing the maximum of the uniformly sampled velocity range accelerates convergence. 
This applies to all tested $n\in\{3,5,7\}$ and both token configurations.

\begin{figure}[t]
  \centering
  \includegraphics[width=\columnwidth]{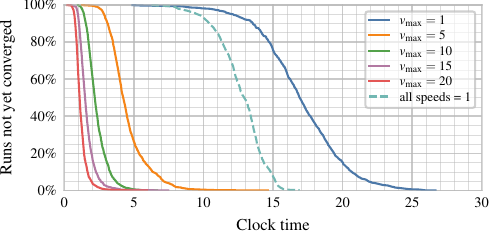}
  \caption{Empirical CCDFs for $n=5$ with \pebAll, varying the maximum velocity.}
  \label{fig:exp_hetero}
  \end{figure}

\paragraph*{\hyperref[rq:pebble_slowdown]{RQ4}} 

The results in \Cref{fig:exp_pebblefactor} show distinct mean convergence time curves for different values of $n$, with the minimum of each curve shifting leftward as $n$ increases.
Specifically, the empirically optimal values are approximately 
$\alpha \approx 0.38,\, 0.35,\, 0.33,\, 0.29$ for $n=3,\, 5,\, 7,\, 9$.
The average convergence time grows super-linearly when deviating from the optimal $\alpha$ in either direction.

\begin{figure}[t]
  \centering
  \includegraphics[width=\columnwidth]{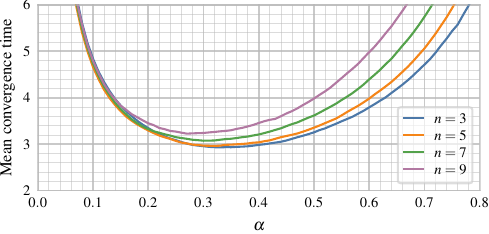}
  \caption{Mean convergence times with \pebAll, varying the slow-down factor $\alpha$. The $\alpha$ inducing the fastest convergence depends on $n$.}
  \label{fig:exp_pebblefactor}
  \end{figure}

\paragraph*{\hyperref[rq:reconvergence]{RQ5}} 

Across all tested perturbation types, the system reliably reconverges to the (possibly new) equilibrium.

After robot addition/removal (\Cref{fig:exp_perturb_add_remove}), the system reliably converges to the correct new equilibrium without coordination. 
Reconvergence after removal is consistently faster than after addition; \pebAll\ reduces convergence times versus \pebOne.
The position jitter sweep (\Cref{fig:exp_perturb_jitterSweep}) shows that mean convergence time increases with jitter amplitude but exhibits at most linear growth; 
we observe no super-linear or catastrophic degradation, even at $25\%$ displacement.
A power-law fit $T(\delta) = t_0 + a\cdot\delta^b$ and bootstrap 95\% confidence intervals confirm sub-linear or linear growth.
Non-zero convergence time at zero jitter reflects detection latency of the convergence criterion (\Cref{def:convergence_criterion}).
Reconvergence after velocity perturbation (\Cref{fig:exp_perturb_speedchange}) is significantly faster than from random starts (Mann--Whitney~U, all $p<0.0125$ after Bonferroni correction).
Consistently with \hyperref[rq:speed_heterogeneity]{RQ3}, $\times 2$ yields faster reconvergence than $\times 0.5$.

\begin{figure}[t]
  \centering
  \includegraphics[width=\columnwidth]{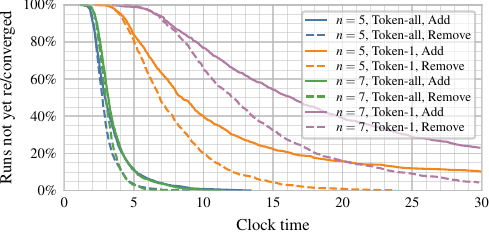}
  \caption{Perturbation of adding/removing one randomly chosen robot from an equilibrium-like state, resulting in $n\in\{5,7\}$ after the change. Reconvergence after removing a robot is faster than reconvergence after adding one; \pebAll \ induces faster reconvergence than \pebOne.}
  \label{fig:exp_perturb_add_remove}
\end{figure}

\begin{figure}[t]
  \centering
  \includegraphics[width=\columnwidth]{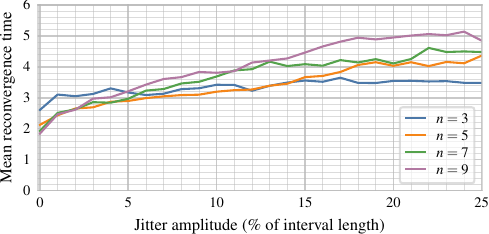}
  \caption{Mean convergence time vs.\ position jitter amplitude (all robots perturbed) with \pebAll, for $n\in\{3,5,7,9\}$. Growth is at most linear; no catastrophic degradation occurs even at $25\%$ displacement.}
  \label{fig:exp_perturb_jitterSweep}
\end{figure}

\begin{figure}[t]
  \centering
  \includegraphics[width=\columnwidth]{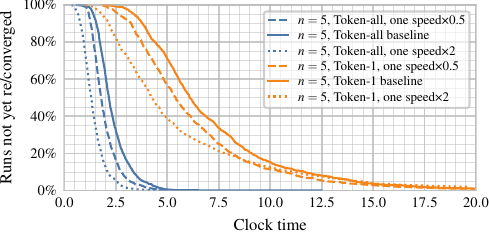}
  \caption{Reconvergence CCDFs after velocity perturbation (factor $0.5$ or $2$) of one robot from equilibrium, for $n=5$, vs.\ random-start baseline. Reconvergence is faster than baseline (Mann--Whitney~U, $p<0.0125$).}
  \label{fig:exp_perturb_speedchange}
\end{figure}

\subsection{Discussion}

\paragraph*{\hyperref[rq:rotational_entry]{RQ1}} 
The observation of the ratio of catch-up events per time step declining to 0 supports the hypothesis that the system always reaches the rotational state, but does not constitute a formal proof that catch-ups never re-emerge. For final conclusiveness, a theoretical argument is needed.

\paragraph*{\hyperref[rq:pebble_placement]{RQ2}} 

\pebAll \ consistently shortens convergence times relative to \pebOne; larger $n$ slows convergence. 
This supports the claim that broader damping distribution improves practical convergence.

\paragraph*{\hyperref[rq:speed_heterogeneity]{RQ3}} 

The results suggest that increasing only the upper bound of a uniform velocity interval accelerates convergence. 
This could be explained by the increased collision rate, due to the higher average velocity, dominating any added disparity in these settings. 
However, the exact dynamic changes and asymptotic behavior with increasing $n$ remain to be explored.

\paragraph*{\hyperref[rq:pebble_slowdown]{RQ4}}

Based on the curves in \Cref{fig:exp_pebblefactor}, we hypothesize that the optimal $\alpha$ decreases as $n$ increases.
Since even for $n=3$ the best value is below $0.4$, selecting a larger $\alpha$ is likely suboptimal for any configuration.

\paragraph*{\hyperref[rq:reconvergence]{RQ5}} 

The system reliably reconverges across all tested perturbation types and magnitudes, supporting operational resilience.
The jitter sweep (\Cref{fig:exp_perturb_jitterSweep}) provides a quantitative robustness characterization: convergence time grows at most linearly with perturbation amplitude, indicating graceful degradation for position inaccuracies or external perturbations.
For velocity perturbations, reconvergence is faster than from random starts, indicating quick recovery from battery drain or load-induced velocity changes.
Robot addition/removal shows robust system reconfiguration behavior in case of individual robot failure or increased team size.

\subsection{Threats to Validity}  

Our criteria for reaching a catch-up-free state and for convergence serve as strong indicators but do not constitute formal guarantees.
The random generation procedure is also unlikely to produce certain edge cases, particularly highly symmetric initial configurations.
To probe the robustness of our criteria, we applied stricter conditions and examined manually crafted scenarios, but these did not reveal counterexamples.
The simulations reported here focus on small numbers of robots, which appear sufficient to reveal the relevant patterns.
Additional exploratory simulations with larger numbers of robots were conducted but are not reported here due to space constraints.

\section{Applications \& Practical Considerations}
\label{sec:applications}

\subsection{Target Application Domains}

\paragraph{Warehouse \& Intralogistics}
Large-scale warehouse robotics has focused initially on open-grid layouts with centralized coordination \cite{coordinating_wurman_2008}, but AMRs increasingly support decentralized decision-making \cite{planning_fragapane_2021,review_keith_2024}. 
As centralized control faces scalability limits with large fleets, tangible decentralization approaches gain attention \cite{review_keith_2024}. Meanwhile, operations along narrow aisles \cite{mobile_borenstein_1995,improving_thomsen_2023} and heterogeneous robot deployments \cite{warehouses_kang_2025}---where dedicated aisle assignments reduce throughput time by up to 26\%---motivate coordination mechanisms suited to confined, one-dimensional paths.

Our bucket brigade approach addresses these challenges: it provides fully decentralized coordination without communication, achieving emergent territorial partitioning where heterogeneous capabilities map to velocity-proportional load balancing. 
Reconvergence after perturbations (\Cref{fig:exp_perturb_add_remove}) enables dynamic fleet adjustments under failures or demand surges---a robustness property desired in current AMR research \cite{planning_fragapane_2021}.

\paragraph{Agricultural Automation}
Agricultural robotics faces environmental challenges (occlusion, harsh conditions, infrastructure limits), and effective commercialization remains challenging due to task complexity \cite{harvesting_bac_2014,agricultural_fountas_2020}. 
Multi-robot greenhouse teams are promising \cite{heterogeneous_roldan_2016}, but coordination remains centralized in most studies, relying on communication infrastructure that cannot always be guaranteed \cite{SONG2026111396}. 
Narrow-aisle operations \cite{SONG2026111396,heterogeneous_roldan_2016} in row-constrained spaces motivate one-dimensional coordination models.

Our bucket brigade provides decentralized coordination without communication infrastructure, also allowing for heterogeneous robot teams. 
Self-recovery (\Cref{fig:exp_perturb_add_remove,fig:exp_perturb_jitterSweep}) addresses fault tolerance needs \cite{heterogeneous_roldan_2016}, while velocity adaptation (\Cref{fig:exp_perturb_speedchange}) models battery/payload variations.

\paragraph{Infrastructure Inspection}
Inspection of linear structures (pipelines, tunnels) has relied on single-robot systems \cite{application_rumson_2021,autonomous_jang_2022}. 
Multi-robot inspection remains mostly limited to open environments \cite{multirobot_liu_2025}, as coordination in confined linear spaces faces challenges where communication infrastructure is unavailable. 
Our decentralized, communication-free approach enables multi-robot teams in such settings.

\subsection{From Model to Physical Systems}

\paragraph{Physical Robot Constraints}
Our idealized model assumes point agents with instantaneous velocity changes and precise collision detection. 
In practice, robots with body radius $r$ and sensor range $d$ operate in the free space between physical boundaries.
The model's unit interval represents this available space, with detection events (via IR/ultrasonic sensors \cite{mobile_borenstein_1995} or compliant bumpers \cite{robot_haddadin_2017}) corresponding to model collisions at separation $d$. 
Acceleration limits mean tokens trigger velocity ramps rather than instantaneous changes; the model needs to be adjusted accordingly. 
Position jitter experiments (\Cref{fig:exp_perturb_jitterSweep}) indicate localization tolerance, however, guarantees outside the idealized system are heuristic. 
Tokens can be implemented via memory (store last collision position) or physical markers (RFID, visual fiducials) for minimal-computation platforms. The latter does not require any position or distance sensing of the robots.
Regarding physical safety, body size, regulations, and wall/marker detection errors reduce usable length and may bias steady borders; safety envelopes and stop accuracy must exceed any hysteresis.

\paragraph{Robustness}
Simulation experiments (Sec.~\ref{sec:simulations}) demonstrate resilience to operational disturbances: robot addition/removal (\Cref{fig:exp_perturb_add_remove}), velocity changes (\Cref{fig:exp_perturb_speedchange}) and position jitter (\Cref{fig:exp_perturb_jitterSweep}) model fleet scaling, battery degradation or payload variations, and localization uncertainty. 
Simulations showing reliable reconvergence after perturbations suggest inherent system resilience. The zero-communication design eliminates infrastructure requirements.

\subsection{Comparison to Alternative Approaches}

\paragraph{Non-zero communication extensions}
Adding communication can simplify the problem but requires additional capabilities: exchanging velocities at collision lets pairs compute equilibrium boundaries directly, yet requires accurate velocity self-knowledge (fragile under battery drain or payload changes), on-board arithmetic, and precise target navigation.
Remote position sharing eliminates physical collisions entirely but demands full infrastructure (localization, radio, waypoint control).
The token mechanism requires none of these—convergence emerges from dynamics rather than computation, requiring only the recognition of crossing the previous collision position, and automatically adapts under perturbations.

\paragraph{Other coordination approaches}
Prior coordination approaches explore components of our problem separately and make different infrastructure trade-offs: 
centralized methods achieve load balancing \cite{spatialtemporal_blad_2023,warehouses_kang_2025} and path optimization \cite{conflictbased_sharon_2015,improved_li_2019} but require infrastructure; 
distributed algorithms \cite{consensus_ren_2005,coverage_cortes_2004,distributed_portugal_2013,cooperative_pasqualetti_2012} improve scalability through message passing; 
heterogeneous-speed patrolling has been optimized via offline scheduling \cite{fence_kawamura_2015}; and communication-free coordination has been demonstrated in foraging \cite{adaptive_lein_2008,Ostergaard2001_EmergentBucketBrigading}.
We complement these approaches by combining communication-free operation, heterogeneous speeds, and persistent patrolling through collision-based bucket brigade coordination.

\section{Conclusion \& Outlook}

We presented an event-based analysis of heterogeneous bucket brigades, showing that undamped collision dynamics are conservative. 
A token mechanism storing a single collision position per robot induces damping and consistent convergence. 
The approach achieves velocity-proportional territory partitioning through collision-based coordination without communication or centralized planning. 
Simulations demonstrate robustness against various perturbation types.

Extensions to graph topologies are underway. Preliminary results show complications (e.g., unreachable subgraphs under certain traversal rules, altered convergence behavior) but suggest using boundary markers to partition graphs into linear segments. 
We also consider slightly more complex but still local coordination mechanisms (e.g., token types with different effect, tokens with $\alpha=1$, ``traffic lights'').

Deployment requires only velocity measurement, acceleration control, collision detection, and single-value memory (token storage) per robot. 
These minimal system requirements position this as a complementary coordination approach for infrastructure-limited environments where communication overhead or centralized control are impractical.

\section*{Acknowledgments}
We thank Aaron T.\ Becker and Erik Demaine for helpful discussions.
OpenAI GPT-5 was used to assist with language editing, code drafting, and standard algebraic operations.

\bibliographystyle{IEEEtran}
\bibliography{lit}

@article{bartholdi1996production,
  title={A production line that balances itself},
  author={Bartholdi III, John J and Eisenstein, Donald D},
	journal={Oper. Res.},
  volume={44},
  number={1},
  year={1996},
  publisher={INFORMS},
  doi={10.1287/opre.44.1.21}
}

@article{bartholdi1999dynamics,
  title={Dynamics of two-and three-worker “bucket brigade” production lines},
  author={Bartholdi III, John J and Bunimovich, Leonid A and Eisenstein, Donald D},
	journal={Oper. Res.},
  volume={47},
  number={3},
  year={1999},
  publisher={INFORMS},
  doi={10.1287/opre.47.3.488}
}

@article{bartholdi2001performance,
  title={Performance of bucket brigades when work is stochastic},
  author={Bartholdi III, John J and Eisenstein, Donald D and Foley, Robert D},
	journal={Oper. Res.},
  volume={49},
  number={5},
  year={2001},
  publisher={INFORMS},
  doi={10.1287/opre.49.5.71.10609}
}

@article{bratcu2009some,
  title={Some new results on the analysis and simulation of bucket brigades (self-balancing production lines)},
  author={Bratcu, Antoneta Iuliana and Dolgui, Alexandre},
	journal={Int. J. Prod. Res.},
  volume={47},
  number={2},
  year={2009},
  publisher={Taylor \& Francis},
  doi={10.1080/00207540802426128}
}

@techreport{bartholdi2004chaos,
  title={Chaos and convergence in bucket brigades with finite walk-back velocities},
  author={Bartholdi, JJ and Eisenstein, Donald D and Lim, Yun Fong},
  institution={Dept. of Industrial Engineering, Georgia Institute of Technology},
  year={2004}
}

@inproceedings{chen2013fence,
  title={On Fence Patrolling by Mobile Agents},
  author={Chen, Ke and Dumitrescu, Adrian and Ghosh, Anirban},
	booktitle={Can. Conf. Comput. Geom. (CCCG)},
  year={2013}
}

@inproceedings{czyzowicz2011boundary,
  title={Boundary patrolling by mobile agents with distinct maximal speeds},
  author={Czyzowicz, Jurek and G{\k{a}}sieniec, Leszek and Kosowski, Adrian and Kranakis, Evangelos},
	booktitle={Eur. Symp. Algorithms (ESA)},
  year={2011},
}

@masterthesis{MensingThesis,
  title        = {Ansätze für {B}ewegungsplanungs-{P}robleme mit dem {Z}iel der verketteten {A}rbeitsteilung},
  author       = {Robert Aron Mensing},
  year         = 2021,
  school       = {Westfälische Wilhelms-Universität Münster},
  type         = {Bachelor's thesis},
  note = {{B}achelor's thesis}
}

@INPROCEEDINGS{Itani1995_MiniRobots,
  author={Itani, H. and Noda, M. and Natsume, M. and Itoh, H. and Tanaka, H. and Hattori, H. and Tanase, H. and Asano, M. and Nohara, T. and Hasegawa, K. and Matsuoka, T. and Matsui, T. and Ando, M. and Nogimori, W. and Naruse, Y.},
	booktitle={Int. Symp. Micro Mach. Hum. Sci. (MHS)}, 
  title={A study of mini-robots bucket brigade system}, 
  year={1995},
  volume={},
  number={},
  doi={10.1109/MHS.1995.494239}
  }

@article{Peng2022_StochasticDiscrete,
author = {Peng Wang and Kai Pan and Zhenzhen Yan and Yun Fong Lim},
title ={Managing Stochastic Bucket Brigades on Discrete Work Stations},
journal = {Prod. Oper. Manag.},
volume = {31},
number = {1},
year = {2022},
doi = {10.1111/poms.13539},
}

@article{Lim2014_Cellular,
author = {Yun Fong Lim and Yue Wu},
title ={Cellular Bucket Brigades on U‐Lines with Discrete Work Stations},

journal = {Prod. Oper. Manag.},
volume = {23},
number = {7},
year = {2014},
doi = {10.1111/poms.12091},
}

@article{lim2011cellular,
  title={Cellular bucket brigades},
  author={Lim, Yun Fong},
	journal={Oper. Res.},
  volume={59},
  number={6},
  year={2011},
  publisher={INFORMS}
}

@article{Lim2009_MaximizingThroughput,
author = {Yun Fong Lim and Kum Khiong Yang},
title ={Maximizing Throughput of Bucket Brigades on Discrete Work Stations},

journal = {Prod. Oper. Manag.},
volume = {18},
number = {1},
year = {2009},
doi = {10.1111/j.1937-5956.2009.01009.x},
}

@article{Bartholdi2006_Trees,
title = {Bucket brigades on in-tree assembly networks},
journal = {Eur. J. Oper. Res.},
volume = {168},
number = {3},
year = {2006},
issn = {0377-2217},
doi = {https://doi.org/10.1016/j.ejor.2004.07.034},
author = {John J. Bartholdi and Donald D. Eisenstein and Yun Fong Lim},
}

@inproceedings{Ostergaard2001_EmergentBucketBrigading,
  title={Emergent bucket brigading: a simple mechanisms for improving performance in multi-robot constrained-space foraging tasks},
  author={Ostergaard, Esben H and Sukhatme, Gaurav S and Matari, Maja J},
	booktitle={Int. Conf. Auton. Agents},
  year={2001}
}

@inproceedings{Klavins2000_ConcurrentRobot,
  title={A formalism for the composition of concurrent robot behaviors},
  author={Klavins, Eric and Koditschek, Daniel E},
	booktitle={IEEE Int. Conf. Robot. Autom. (ICRA)},
  year={2000},
}

@article{SONG2026111396,
title = {Design of robot navigation system under spatial constraints in shiitake mushroom fruiting room},
journal = {Comput. Electron. Agric.},
volume = {243},
year = {2026},
doi = {https://doi.org/10.1016/j.compag.2025.111396},
author = {Hualu Song and Kangkang Qi and Wenjie Feng and Yuanjie Mu and Jun Li and Xiangyu Lv and Zhichao Liang and Fengyun Wang},
}

@article{adaptive_lein_2008,
	title = {Adaptive multi-robot bucket brigade foraging},
	author = {Lein, Adam and Vaughan, Richard},
	journal = {Artificial Life},
	year = {2008},
	researchRabbitId = {9c591c59-d459-4dbf-b8f0-bbd99b0f2a32}
}

@article{fence_kawamura_2015,
	title = {Fence patrolling by mobile agents with distinct speeds},
	doi = {10.1007/S00446-014-0226-3},
	author = {Kawamura, Akitoshi and Kobayashi, Yusuke},
	journal = {Distributed Computing},
	year = {2015},
	researchRabbitId = {fd7cc196-0290-442f-a14d-a8bb425b464d}
}

@article{warehouses_kang_2025,
	title = {Warehouses with heterogeneous robots collaboration: operational policies and performance analysis},
	doi = {10.1080/00207543.2025.2513576},
	author = {Kang, Yuexin and Wang, Rong and Qin, Zhizhen and Yang, Peng and Yan, Yimo},
	journal = {Int. J. Prod. Res.},
	year = {2025},
	researchRabbitId = {8e9a86f0-98b8-4f92-9b10-aaaf02ae8809}
}

@article{cooperative_pasqualetti_2012,
	title = {Cooperative Patrolling via Weighted Tours: Performance Analysis and Distributed Algorithms},
	doi = {10.1109/TRO.2012.2201293},
	author = {Pasqualetti, Fabio and Durham, Joseph W. and Bullo, Francesco},
	journal = {IEEE Trans. Robot.},
	year = {2012},
	researchRabbitId = {455050d9-2695-4553-bfc8-5623a0715eb5}
}

@article{improved_li_2019,
	title = {Improved Heuristics for Multi-Agent Path Finding with Conflict-Based Search: Preliminary Results},
	doi = {10.1609/SOCS.V10I1.18481},
	author = {Li, Jiaoyang and Felner, Ariel and Boyarski, Eli and Koenig, Sven and Ma, Hang},
	journal = {Symp. Combinatorial Search},
	year = {2019},
	researchRabbitId = {4a814167-ab47-4af1-9eff-e8c86cb1e9eb}
}

@article{autonomous_jang_2022,
	title = {Autonomous Navigation of In-Pipe Inspection Robot Using Contact Sensor Modules},
	doi = {10.1109/TMECH.2022.3162192},
	author = {Jang, Heesik and Kim, Tae Yu and Lee, Ye Chan and Song, Yong Heon and Choi, Hyouk Ryeol},
	journal = {IEEE/ASME Trans. Mechatronics},
	year = {2022},
	researchRabbitId = {1282eeca-5a26-40e0-b690-621a074da859}
}

@article{application_rumson_2021,
	title = {The application of fully unmanned robotic systems for inspection of subsea pipelines},
	doi = {10.1016/J.OCEANENG.2021.109214},
	author = {Rumson, Alexander G.},
	journal = {Ocean Engineering},
	year = {2021},
	researchRabbitId = {7b7c3a52-efdb-4c38-a032-0e632845b4b4}
}

@article{review_keith_2024,
	title = {Review of Autonomous Mobile Robots for the Warehouse Environment},
	doi = {10.48550/ARXIV.2406.08333},
	author = {Keith, Russell and La, Hung Manh},
	journal = {arXiv.org},
	year = {2024},
	researchRabbitId = {4d0e57b9-5971-4522-92a0-142f4ee88c51}
}

@article{coordinating_wurman_2008,
	title = {Coordinating Hundreds of Cooperative, Autonomous Vehicles in Warehouses},
	doi = {10.1609/AIMAG.V29I1.2082},
	author = {Wurman, Peter R. and D’Andrea, Raffaello and Mountz, Mick},
	journal = {AI Mag.},
	year = {2008},
	researchRabbitId = {6b753a4c-6aa5-496f-8a4a-fefbc3df7462}
}

@article{planning_fragapane_2021,
	title = {Planning and control of autonomous mobile robots for intralogistics: Literature review and research agenda},
	doi = {10.1016/J.EJOR.2021.01.019},
	author = {Fragapane, Giuseppe and Koster, René de and Sgarbossa, Fabio and Strandhagen, Jan Ola},
	journal = {Eur. J. Oper. Res.},
	year = {2021},
	researchRabbitId = {1f239a20-67c8-4bfd-add6-c49ac018a087}
}

@article{coverage_cortes_2004,
  title={Coverage control for mobile sensing networks},
  author={Cortes, Jorge and Martinez, Sonia and Karatas, Timur and Bullo, Francesco},
	journal={IEEE Trans. Robot. Autom.},
  volume={20},
  number={2},
  year={2004},
  publisher={IEEE}
}

@article{conflictbased_sharon_2015,
	title = {Conflict-based search for optimal multi-agent pathfinding},
	doi = {10.1016/J.ARTINT.2014.11.006},
	author = {Sharon, Guni and Stern, Roni and Felner, Ariel and Sturtevant, Nathan},
	journal = {Artif. Intell.},
	year = {2015},
	researchRabbitId = {0ca063d0-6ae8-4463-ab60-a23ac369f354}
}

@article{consensus_ren_2005,
	title = {Consensus seeking in multiagent systems under dynamically changing interaction topologies},
	doi = {10.1109/TAC.2005.846556},
	author = {Ren, Wei and Beard, Randal W.},
	journal = {IEEE Trans. Autom. Control},
	year = {2005},
	researchRabbitId = {42d73015-7ffa-4aad-bdae-9c8121501ddb}
}

@article{heterogeneous_roldan_2016,
	title = {Heterogeneous Multi-Robot System for Mapping Environmental Variables of Greenhouses},
	doi = {10.3390/S16071018},
	author = {Roldan, Juan Jesus and Garcia-Aunon, Pablo and Garzón, Mario and León, J. de la Fuente and Cerro, Jaime del and Barrientos, Antonio},
	journal = {Sensors},
	year = {2016},
	pubmedId = {https://pubmed.ncbi.nlm.nih.gov/27376297},
	researchRabbitId = {29150aab-b33c-4c0c-b171-68ee24ea08cb}
}

@article{harvesting_bac_2014,
	title = {Harvesting Robots for High‐value Crops: State‐of‐the‐art Review and Challenges Ahead},
	doi = {10.1002/ROB.21525},
	author = {Bac, C.W. and Henten, E.J. van and Hemming, J. and Edan, Yael},
	journal = {J. Field Robotics},
	year = {2014},
	researchRabbitId = {45ca545e-96f5-4ad3-9176-661108050ce9}
}

@article{agricultural_fountas_2020,
	title = {Agricultural Robotics for Field Operations.},
	doi = {10.3390/S20092672},
	author = {Fountas, Spyros and Mylonas, Nikos and Malounas, Ioannis and Rodias, Efthymios and Santos, Christoph Hellmann and Pekkeriet, E.J.},
	journal = {Sensors},
	year = {2020},
	pubmedId = {https://pubmed.ncbi.nlm.nih.gov/32392872},
	researchRabbitId = {4d15ccc8-3976-4c99-bb2c-87a39e3257b1}
}

@article{robot_haddadin_2017,
	title = {Robot Collisions: A Survey on Detection, Isolation, and Identification},
	doi = {10.1109/TRO.2017.2723903},
	author = {Haddadin, Sami and Luca, Alessandro De and Albu‐Schäffer, Alin},
	journal = {IEEE Trans. Robot.},
	year = {2017},
	researchRabbitId = {97bd1f96-0e41-4b0c-b7e8-d077539ae333}
}

@article{multirobot_liu_2025,
	title = {Multi-robot cooperative inspection planning of substation based on genetic algorithm and deep reinforcement learning},
	doi = {10.1109/ACCESS.2025.3598345},
	author = {Liu, Jie and Wang, Jingsheng and Peng, Jiajun and Yuan, Hong and Shao, Tingting},
	journal = {IEEE Access},
	year = {2025},
	researchRabbitId = {50ff4ef6-a69f-4c97-91e1-d9e2f6fb2e47}
}

@article{distributed_portugal_2013,
	title = {Distributed multi-robot patrol: A scalable and fault-tolerant framework},
	doi = {10.1016/J.ROBOT.2013.06.011},
	author = {Portugal, David and Rocha, Rui P.},
	journal = {Robot. Auton. Syst.},
	year = {2013},
	researchRabbitId = {61b91adf-e522-4b7e-941c-c92374e2df94}
}

@inproceedings{mobile_borenstein_1995,
  title={Mobile robot navigation in narrow aisles with ultrasonic sensors},
  author={Borenstein, Johann and Wehe, David and Feng, Liqiang and Koren, Yoram},
  booktitle={ANS Topical Meeting on Robotics and Remote Systems},
  volume={12},
  year={1995}
}

@article{spatialtemporal_blad_2023,
	title = {Spatial-Temporal Load Balancing and Coordination of Multi-Robot Stations},
	doi = {10.1109/TASE.2022.3214567},
	author = {Åblad, Edvin and Spensieri, Domenico and Bohlin, Robert and Carlson, Johan S. and Strömberg, Ann-Brith},
	journal = {IEEE Trans. Autom. Sci. Eng.},
	year = {2023},
	researchRabbitId = {523ce0de-6e0f-40d6-97dc-18f69c9afee6}
}

@article{improving_thomsen_2023,
	title = {Improving Throughput of Mobile Robots in Narrow Aisles},
	doi = {10.5220/0011717500003417},
	author = {Thomsen, Simon Francis and Davidsen, Martin and Naik, Lakshadeep and Kollakidou, Avgi and Bodenhagen, Leon and Krüger, Norbert},
	journal = {VISIGRAPP},
	year = {2023},
	researchRabbitId = {418b83f8-4e44-44db-9655-6907450f1e8e}
}

\end{document}